\documentclass{article}

\usepackage{times}
\usepackage{graphicx} 

\usepackage{algorithm}
\usepackage{algorithmic}

\usepackage{natbib}

\usepackage{hyperref}

\usepackage[accepted]{icml2017}

\usepackage{helvet}
\usepackage{courier}
\usepackage{amsmath}
\usepackage{amsfonts} 
\usepackage{mathtools} 
\usepackage{amssymb}
\usepackage[usenames,dvipsnames,svgnames,table]{xcolor} 
\mathtoolsset{showonlyrefs} 
\usepackage{amsthm}	

\newtheorem{thm}{Theorem}[]

\newtheorem{lemma}{Lemma}[]

\newtheorem{cor}{Corollary}[]
\usepackage{theoremref} 
\usepackage{amsxtra} 
\definecolor{dark-blue}{rgb}{0,0,0.75}
\hypersetup{
    colorlinks, linkcolor={dark-blue},
    citecolor={dark-blue}, urlcolor={dark-blue}
}

\usepackage{enumitem}
\setlist{nolistsep}

\usepackage{subcaption}

\newcommand{\btheta}{{\boldsymbol\theta}}

\icmltitlerunning{Data-Efficient Policy Evaluation Through Behavior Policy Search}

\begin{document}

\twocolumn[
\icmltitle{Data-Efficient Policy Evaluation Through Behavior Policy Search}

\icmlsetsymbol{equal}{*}

\begin{icmlauthorlist}
\icmlauthor{Josiah P. Hanna}{uta}
\icmlauthor{Philip S. Thomas}{umass,cmu}
\icmlauthor{Peter Stone}{uta}
\icmlauthor{Scott Niekum}{uta}
\end{icmlauthorlist}

\icmlaffiliation{uta}{The University of Texas at Austin, Austin, Texas, USA}
\icmlaffiliation{umass}{The University of Massachusetts, Amherst, Massachusetts, USA}
\icmlaffiliation{cmu}{Carnegie Mellon University, Pittsburgh, Pennsylvania, USA}

\icmlcorrespondingauthor{Josiah P. Hanna}{jphanna@cs.utexas.edu}

\icmlkeywords{off-policy evaluation, importance-sampling}

\vskip 0.3in
]

\printAffiliationsAndNotice{}

\begin{abstract}
We consider the task of evaluating a policy for a \textit{Markov decision process} (MDP). 
The standard unbiased technique for evaluating a policy is 
to deploy the policy and observe its performance.
We show that the data collected from deploying a different policy, commonly called the \textit{behavior policy}, can be used to produce unbiased estimates with lower mean squared error than this standard technique.
We derive an analytic expression for the \textit{optimal behavior policy}---the behavior policy that minimizes the mean squared error of the resulting estimates. 
Because this expression depends on terms that are unknown in practice, we propose a novel policy evaluation sub-problem, \textit{behavior policy search}: searching for a behavior policy that reduces mean squared error. 
We present a behavior policy search algorithm and empirically demonstrate its effectiveness in lowering the mean squared error of policy performance estimates.
%
\end{abstract}


\section{Introduction}

Many sequential decision problems, including diabetes treatment \cite{bastani2014model}, digital marketing \cite{theocharous2015personalized}, and robot control \cite{lillicrap2015continuous}, are modeled as \textit{Markov decision processes} (MDPs) and solved using \emph{reinforcement learning} (RL) algorithms. 
One important problem when applying RL to real problems is \textit{policy evaluation}. 
The goal in policy evaluation is to estimate the expected \textit{return} (sum of rewards) produced by a policy. 
We refer to this policy as the \emph{evaluation policy}, $\pi_e$. 
The standard policy evaluation approach is to repeatedly deploy $\pi_e$ and average the resulting returns. 
While this na\"{i}ve Monte Carlo estimator is unbiased, it may have high variance.

Methods that evaluate $\pi_e$ while selecting actions according to $\pi_e$ are termed \emph{on-policy}.
Previous work has addressed variance reduction for on-policy returns \cite{Zinkevich2006,White2009,Veness2011}. 
An alternative approach is to estimate the performance of $\pi_e$ while following a different, \emph{behavior policy}, $\pi_b$. 
Methods that evaluate $\pi_e$ with data generated from $\pi_b$ are termed \emph{off-policy}. 
\textit{Importance sampling} (IS) is one standard approach for using off-policy data in RL. 
IS reweights returns observed while executing $\pi_b$ such that they are unbiased estimates of the performance of $\pi_e$. 

Presently, IS is usually used when off-policy data is already available or when executing $\pi_e$ is impractical. 
If $\pi_b$ is not chosen carefully, IS often has high variance \cite{thomas2015off-policy}.
For this reason, an implicit assumption in the RL community has generally been that on-policy evaluation is more accurate when it is feasible. 
However, IS can also be used for variance reduction when done with an appropriately selected distribution of returns \cite{Hammersley1964}.
While IS-based variance reduction has been explored in RL, this prior work has required knowledge of the environment's transition probabilities and remains on-policy \cite{desai2001simulation,frank2008reinforcement,ciosek2017offer}.
In contrast to this earlier work, we show how careful selection of the behavior policy can lead to lower variance policy evaluation than using the evaluation policy and do not require knowledge of the environment's transition probabilities.

%
%
%
%

In this paper, we formalize the selection of $\pi_b$ as the \emph{behavior policy search} problem.
We introduce a method for this problem that adapts the policy parameters of $\pi_b$ with gradient descent on the variance of importance-sampling. 
Empirically we demonstrate behavior policy search with our method lowers the mean squared error of estimates compared to on-policy estimates.
To the best of our knowledge, this work is the first to propose adapting the behavior policy to obtain better policy evaluation in RL.
Furthermore we present the first method to address this problem.

\section{Preliminaries}

This section details the policy evaluation problem setting, the Monte Carlo and Advantage Sum on-policy methods, and importance-sampling for off-policy evaluation.

\subsection{Background}

We use notational standard MDPNv1 \citep{Thomas2015notation}, and for simplicity, we assume that $\mathcal{S},\mathcal{A},$ and $R$ are finite.\footnote{The  methods, and theoretical results discussed in this paper are applicable to both finite and infinite $\mathcal{S},\mathcal{A}$ and $R$ as well as \textit{partially-observable} Markov decision processes.} Let $H\coloneqq(S_0,A_0,R_0,S_1,\dotsc,S_{L},A_{L},R_{L})$ be a \textit{trajectory} and $g(H)\coloneqq \sum_{t=0}^{L} \gamma^t R_t$ be the \textit{discounted return} of trajectory $H$. Let $\rho(\pi)\coloneqq \mathbf{E}[g(H) | H \sim \pi]$ be the expected discounted return when the stochastic policy $\pi$ is used from $S_0$ sampled from the initial state distribution. 
In this work, we consider parameterized policies, $\pi_\btheta$, where the distribution over actions is determined by the vector $\btheta$.
We assume that the transitions and reward function are unknown and that $L$ is finite. 

We are given an \textit{evaluation policy}, $\pi_e$, for which we would like to estimate $\rho(\pi_e)$. 
We assume there exists a policy parameter vector $\btheta_e$ such that $\pi_e = \pi_{\btheta_e}$ and that this vector is known.
We consider an incremental setting where,
at iteration $i$, we sample a single trajectory $H_i$ with a policy $\pi_{\btheta_i}$ and add $\{H_i,\btheta_i\}$ to a set $\mathcal{D}$.
We use $\mathcal{D}_i$ to denote the set at iteration $i$.
Methods that always (i.e., $\forall i$) choose $\btheta_i = \btheta_e$ are on-policy; otherwise, the method is off-policy.
A policy evaluation method, $\operatorname{PE}$, uses all trajectories in $\mathcal{D}_i$ to estimate $\rho(\pi_e)$.
Our goal is to design a policy evaluation algorithm that produces estimates of $\rho(\pi_e)$ that have low \textit{mean squared error} (MSE). 
Formally, the goal of policy evaluation with PE is to minimize $(\operatorname{PE}(\mathcal{D}_i) - \rho(\pi_e))^2$.
While other measures of policy evaluation accuracy could be considered, we follow earlier work in using MSE (e.g., \cite{thomas2016data-efficient,precup2000eligibility}).

We focus on unbiased estimators of $\rho(\pi_e)$.
While biased estimators (e.g., bootstrapping methods \cite{sutton1998reinforcement}, approximate models \cite{kearns2002near}, etc.) can sometimes produce lower MSE estimates they are problematic for high risk applications requiring confidence intervals.
For unbiased estimators, minimizing variance is equivalent to minimizing MSE.

\subsection{Monte-Carlo Estimates}

Perhaps the most commonly used policy evaluation method is the \textit{on-policy Monte-Carlo} (MC) estimator. 
The estimate of $\rho(\pi_e)$ at iteration $i$ is the average return:
\[
\operatorname{MC}(\mathcal{D}_i) \coloneqq \frac{1}{i+1}\sum_{j=0}^i \sum_{t=0}^L \gamma^t R_t = \frac{1}{i + 1}\sum_{j=0}^i g(H_j).
\]

This estimator is unbiased and strongly consistent given mild assumptions.\footnote{Being a strongly consistent estimator of $\rho(\pi_e)$ means that $\Pr\left(\displaystyle\lim_{i\rightarrow \infty} \operatorname{MC}(\mathcal{D}_i )= \rho(\pi_e)\right) = 1$.  If $\rho(\pi_e)$ exists, $\operatorname{MC}$ is strongly consistent by the Khintchine Strong law of large numbers \cite{sen1993large}.} However, this method can have high variance.


\subsection{Advantage Sum Estimates}\label{sec:ase}

Like the Monte-Carlo estimator, the \textit{advantage sum} (ASE) estimator selects $\btheta_i=\btheta_e$ for all $i$. However, it introduces a control variate to reduce the variance without introducing  bias. 
This control variate requires an approximate model of the MDP to be provided.
Let the reward function of this model be given as $\hat{r}(s,a)$.
Let $\hat{q}^{\pi_e}(s_t,a_t) = \mathbf{E}[\sum_{t'=t}^L \gamma^{t'}\hat{r}(s_{t'},a_{t'}) ]$ and $\hat{v}^{\pi_e}(s_t) = \mathbf{E}[\hat{q}^{\pi_e}(s_t,a_t) | a_t \sim \pi_e]$, i.e., the action-value function and state-value function of $\pi_e$ in this approximate model.
Then, the advantage sum estimator is given by:
$$
\operatorname{AS}(D_i) \coloneqq \frac{1}{i+1}\sum_{j=0}^i \sum_{t=0}^L \gamma^t (R_t-\hat q^{\pi_e}(S_t,A_t) + \hat v^{\pi_e}(S_t)).
$$

Intuitively, ASE is replacing part of the randomness of the Monte Carlo return with the known expected return under the approximate model.
Provided $q^{\pi_e}(S_t,A_t) - \hat{v}^{\pi_e}(S_t)$ is sufficiently correlated with $R_t$, the variance of ASE is less than that of MC.

Notice that, like the MC estimator, the ASE estimator is \textit{on-policy}, in that the behavior policy is always the policy that we wish to evaluate. 
Intuitively it may seems like this choice should be optimal. 
However, we will show that it is not---selecting behavior policies that are different from the evaluation policy can result in estimates of $\rho(\pi_e)$ that have lower variance.

\subsection{Importance Sampling}

\textit{Importance Sampling} is a method for reweighting returns from a \emph{behavior} policy, $\btheta$, such that they are unbiased returns from the \emph{evaluation} policy.
In RL, the re-weighted IS return of a trajectory, $H$, sampled from $\pi_\btheta$ is:
$$
\operatorname{IS}(H,\btheta) \coloneqq g(H) \prod_{t=0}^L \frac{\pi_e(S_t|A_t)}{\pi_{\btheta}(S_t|A_t)}.
$$

The IS off-policy estimator is then a Monte Carlo estimate  of $\mathbf{E}\left[IS(H,\btheta) \middle| H \sim \pi_\btheta\right]$:
\[ \operatorname{IS}(\mathcal{D}_i) \coloneqq \frac{1}{i + 1}\sum_{j=0}^i \operatorname{IS}(H_j,\btheta_j).\]

In RL, importance sampling allows off-policy data to be used as if it were on-policy. 
In this case the variance of the IS estimate is often much worse than the variance of on-policy MC estimates because the behavior policy is not chosen to minimize variance, but is a policy that is dictated by circumstance.

\section{Behavior Policy Search}

Importance sampling was originally intended as a variance reduction technique for Monte Carlo evaluation \citep{Hammersley1964}.
When an evaluation policy rarely samples trajectories with high magnitude returns a Monte Carlo evaluation will have high variance.
If a behavior policy can increase the probability of observing such trajectories then the off-policy IS estimate will have lower variance than an on-policy Monte Carlo estimate.
In this section we first describe the theoretical potential for variance reduction with an appropriately selected behavior policy.
In general this policy will be unknown.
Thus, we propose a policy evaluation subproblem --- the behavior policy search problem --- solutions to which will adapt the behavior policy to provide lower mean squared error policy performance estimates.
To the best of our knowledge, we are the first to propose behavior policy adaptation for policy evaluation.

\subsection{The Optimal Behavior Policy}\label{sec:opt}

An appropriately selected behavior policy can lower variance to zero.
While this fact is generally known for importance-sampling, we show here that this policy exists for any MDP and evaluation policy under two restrictive assumptions: all returns are positive and the domain is deterministic.
In the following section we describe how an initial policy can be adapted towards the optimal behavior policy even when these conditions fail to hold.

Let $w_\pi(H) \coloneqq \prod_{t=0}^L \pi(A_t|S_t)$. 
Consider a behavior policy $\pi_b^\star$ such that for any trajectory, $H$:
\[ \rho(\pi_e) = \operatorname{IS}(H,\pi_b^\star) = g(H) \frac{w_{\pi_e}(H)}{w_{\pi_b^\star}(H)}  .\]
Rearranging the terms of this expressions yields:
\begin{align}
 w_{\pi_b^\star}(H) = g(H)\frac{w_{\pi_e}(H)}{\rho(\pi_e)}.
\end{align}
Thus, if we can select $\pi_b^\star$ such that the probability of observing any $H \sim \pi_b^\star$ is $\frac{g(H)}{\rho(\pi_e)}$ times the likelihood of observing $H \sim \pi_e$ then the $\operatorname{IS}$ estimate has zero MSE with only a single sampled trajectory.
Regardless of $g(H)$, the importance-sampled return will equal	 $\rho(\pi_e)$.

Furthermore, the policy $\pi_b^\star$ exists within the space of all feasible stochastic policies.
Consider that a stochastic policy can be viewed as a mixture policy over time-dependent (i.e., action selection depends on the current time-step) deterministic policies.
For example, in an MDP with one state, two actions and a horizon of $L$ there are $2^L$ possible time-dependent deterministic policies, each of which defines a unique sequence of actions.
We can express any evaluation policy as a mixture of these deterministic policies.
The optimal behavior policy $\pi_b^\star$ can be expressed similarly and thus the optimal behavior policy exists.

Unfortunately, the optimal behavior policy depends on the unknown value $\rho(\pi_e)$ as well as the unknown reward function $R$ (via $g(H)$). 
Thus, while there exists an optimal behavior policy for $\operatorname{IS}$ -- which is not $\pi_e$ -- in practice we cannot analytically determine $\pi_b^\star$.
Additionally, $\pi_b^\star$ may not be representable by any $\btheta$ in our policy class.

\subsection{The Behavior Policy Search Problem}

Since the optimal behavior policy cannot be analytically determined, we instead propose the behavior policy search (BPS) problem for finding $\pi_b$ that lowers the MSE of estimates of $\rho(\pi_e)$.
A BPS problem is defined by the inputs:
\begin{enumerate}
\item An evaluation policy $\pi_e$ with policy parameters $\btheta_e$.
\item An off-policy policy evaluation algorithm, $\operatorname{OPE}(H,\btheta)$, that takes a trajectory, $H \sim \pi_\btheta$, or, alternatively, a set of trajectories, and returns an estimate of $\rho(\pi_e)$.
\end{enumerate}
A BPS solution is a policy, $\pi_{\btheta_b}$ such that off-policy estimates with $\operatorname{OPE}$ have lower MSE than on-policy estimates.
Methods for this problem are BPS algorithms.

Recall we have formalized policy evaluation within an incremental setting where one trajectory for policy evaluation is generated each iteration.
At the $i^\text{th}$ iteration, a BPS algorithm selects a behavior policy that will be used to generate a trajectory, $H_i$. 
The policy evaluation algorithm, OPE, then estimates $\rho(\pi_e)$ using trajectories in $\mathcal{D}_i$.
Naturally, the selection of the behavior policy depends on how OPE estimates $\rho(\pi_e)$.

In a BPS problem, the $i^\text{th}$ iteration proceeds as follows. 
First, given all of the past behavior policies, $\{\btheta_i\}_{i=0}^{i-1}$, and the resulting trajectories, $\{H_i\}_{i=0}^{i-1}$, the BPS algorithm must select $\btheta_i$. 
The policy $\pi_{\btheta_i}$ is then run for one episode to create the trajectory $H_i$. 
Then the BPS algorithm uses OPE to estimate $\rho(\pi_e)$ given the available data, $D_i\coloneqq \{(\btheta_i, H_i)\}_{i=0}^i$.
In this paper, we consider the one-step problem of selecting $\btheta_i$ and estimating $\rho(\pi_e)$ at iteration $i$ in a way that minimizes MSE. 
That is, we do not consider how our selection of $\btheta_i$ will impact our future ability to select an appropriate $\btheta_j$ for $j > i$ and thus to produce more accurate estimates in the future.


One natural question is: if we are given a limit on the number of trajectories that can be sampled, is it better to  ``spend" some of our limited trajectories on BPS instead of using on-policy estimates?
Since each $OPE(H_i,\btheta_i)$ is an unbiased estimator of $\rho(\pi_e)$, we can use all sampled trajectories to compute $\operatorname{OPE}(\mathcal{D}_i)$.
Provided for all iterations, $\operatorname{Var}[\operatorname{OPE}(H,\btheta_i)] \leq Var[MC]$ then, in expectation, a BPS algorithm will always achieve lower MSE than MC, showing that it is, in fact, worthwhile to do so.
This claim is supported by our empirical study.


\section{Behavior Policy Gradient Theorem}

We now introduce our primary contributions: an analytic expression for the gradient of the mean squared error of the $\operatorname{IS}$ estimator and a stochastic gradient descent algorithm that adapts $\btheta$ to minimize the MSE between the IS estimate and $\rho(\pi_e)$.
Our algorithm --- \textbf{B}ehavior \textbf{P}olicy \textbf{G}radient (BPG) --- begins with on-policy estimates and adapts the behavior policy with gradient descent on the MSE with respect to $\btheta$.
The gradient of the MSE with respect to the policy parameters is given by the following theorem:

\begin{thm} \label{theorem}
{\small
\[ \frac{\partial}{\partial \btheta}\operatorname{MSE}[\operatorname{IS}(H, \btheta)] = \mathbf{E} \left [ -\operatorname{IS}(H,\btheta)^2 \sum_{t=0}^{L} \frac{\partial}{\partial \btheta} \log \pi_{\btheta}(A_t|S_t)  \right] \]
}
where the expectation is taken over $H \sim \pi_\btheta$.
\end{thm}
\begin{proof}\renewcommand{\qedsymbol}{}
Proofs for all theoretical results are included in Appendix A.
\end{proof}


BPG uses stochastic gradient descent in place of exact gradient descent: replacing the intractable expectation in Theorem 1 with an unbiased estimate of the true gradient.
In our experiments, we sample a batch, $\mathcal{B}_i$, of $k$ trajectories with $\pi_{\btheta_i}$ to lower the variance of the gradient estimate at iteration $i$.
In the BPS setting, sampling a batch of trajectories is equivalent to holding $\btheta$ fixed for $k$ iterations and then updating $\btheta$ with the $k$ most recent trajectories used to compute the gradient estimate.

Full details of BPG are given in Algorithm \ref{alg:bpg}.
At iteration $i$, BPG samples a batch, $\mathcal{B}_i$, of $k$ trajectories and adds $\{(\btheta_i,H_i)_{i=0}^k\}$ to a data set $\mathcal{D}$ (Lines 4-5).
Then BPG updates $\btheta$ with an empirical estimate of Theorem \ref{theorem} (Line 6).
After $n$ iterations, the $\operatorname{BPG}$ estimate of $\rho(\pi_e)$ is $\operatorname{IS}(\mathcal{D}_n)$ as defined in Section 2.4.

\begin{algorithm}

{\fontsize{9}{9}\selectfont
\begin{algorithmic}[1]

\STATE $\btheta_0 \leftarrow \btheta_e$
\STATE $\mathcal{D}_0 = \{\}$
\FORALL{$i \in 0...n$}
\STATE $\mathcal{B}_i=$ Sample $k$ trajectories $H \sim \pi_{\btheta_i}$
\STATE $\mathcal{D}_{i+1} = \mathcal{D}_i \cup \mathcal{B}_i$
\STATE $\btheta_{i+1} = \btheta_i + \frac{\alpha_i}{k} \displaystyle\sum_{H \in \mathcal{B}} \operatorname{IS}(H,\btheta)^2 \displaystyle\sum_{t=0}^L \frac{\partial}{\partial\btheta} \log \pi_{\btheta_i}(A_t|S_t)$

\ENDFOR

\STATE \textbf{Return} $\btheta_n$, $\operatorname{IS}(\mathcal{D}_n)$

\end{algorithmic}
}
\caption{\textbf{Behavior Policy Gradient} \newline \textbf{Input:} Evaluation policy parameters, $\btheta_e$, batch size $k$, a step-size for each iteration, $\alpha_i$, and number of iterations $n$.
\newline \textbf{Output:} Final behavior policy parameters $\btheta_n$ and the IS estimate of $\rho(\pi_e)$ using all sampled trajectories.}
\label{alg:bpg}
\end{algorithm}

Given that the step-size, $\alpha_i$, is consistent with standard gradient descent convergence conditions, $\operatorname{BPG}$ will converge to a behavior policy that locally minimizes the variance \cite{bertsekas2000gradient}. 
At best, $\operatorname{BPG}$ converges to the globally optimal behavior policy \emph{within the parameterization of $\pi_e$}.
Since the parameterization of $\pi_e$ determines the class of representable distributions it is possible that the theoretically optimal behavior policy is unrepresentable under this parameterization.
Nevertheless, a suboptimal behavior policy still yields better estimates of $\rho(\pi_e)$, provided it decreases variance compared to on-policy returns.

\subsection{Control Variate Extension}

In cases where an approximate model is available, we can further lower variance adapting the behavior policy of the \emph{doubly robust} estimator \cite{jiang2016doubly, thomas2016data-efficient}.
Based on a similar intuition as the Advantage Sum estimator (Section \ref{sec:ase}), the Doubly Robust (DR) estimator uses the value functions of an approximate model as a control variate to lower the variance of importance-sampling.\footnote{DR lowers the variance of \emph{per-decision} importance-sampling which importance samples the per time-step reward.}
We show here that we can adapt the behavior policy to lower the mean squared error of DR estimates.
We denote this new method DR-BPG for \textbf{D}oubly \textbf{R}obust \textbf{B}ehavior \textbf{P}olicy \textbf{G}radient.

Let $w_{\pi,t}(H) = \prod_{i=0}^t \pi(A_t|S_t)$ and recall that $\hat{v}^{\pi_e}$ and $\hat{q}^{\pi_e}$ are the state and action value functions of $\pi_e$ in the approximate model.
The DR estimator is:
{\small
\[
\operatorname{DR}(H,\btheta)\coloneqq \hat{v}(S_0) + \sum_{t=0}^L \frac{w_{\pi_e,t}}{w_{\pi_\btheta,t}}(R_t - \hat{q}^{\pi_e}(S_t,A_t) + \hat{v}^{\pi_e}(S_{t+1})).
\]
}


We can reduce the mean squared error of DR with gradient descent using unbiased estimates of the following corollary to Theorem 1:
\begin{cor}
{
\small
\begin{multline*}
\frac{\partial}{\partial\btheta}\operatorname{MSE}\left[DR(H,\btheta)\right] = \mathbf{E} [( \operatorname{DR}(H,\btheta)^2 \sum_{t=0}^L \frac{\partial}{\partial\btheta}\log \pi_\btheta(A_t|S_t) \\  -   2 \operatorname{DR}(H,\btheta)(\sum_{t=0}^L \gamma^t \delta_t\frac{w_{\pi_e,t}}{w_{\btheta,t}} \sum_{i=0}^t \frac{\partial}{\partial\btheta} \log \pi_\btheta (A_i|S_i))  ]
\end{multline*}
}
where $\delta_t = R_t - \hat{q}(S_t,A_t) + \hat{v}(S_{t+1})$ and the expectation is taken over $H \sim \pi_\btheta$.
\end{cor}

The first term of $\frac{\partial}{\partial\btheta}MSE$ is analogous to the gradient of the importance-sampling estimate with $\operatorname{IS}(H,\btheta)$ replaced by $\operatorname{DR}(H,\btheta)$.
The second term accounts for the covariance of the DR terms.

AS and DR both assume access to a model, however, they make no assumption about where the model comes from except that it must be independent of the trajectories used to compute the final estimate.
In practice, AS and DR perform best when all trajectories are used to estimate the model and then used to estimate $\rho(\pi_e)$ \cite{thomas2016data-efficient}.
However, for DR-BPG, changes to the model change the surface of the MSE objective we seek to minimize and thus DR-BPG will only converge once the model stops changing.
In our experiments, we consider both a changing and a fixed model.



\subsection{Connection to REINFORCE}

BPG is closely related to existing work in policy gradient RL (c.f., \cite{sutton2000policy}) and we draw connections between one such method and BPG to illustrate how BPG changes the distribution of trajectories.
REINFORCE \cite{williams1992simple} attempts to maximize $\rho(\pi_\btheta)$ through gradient ascent on $\rho(\pi_\btheta)$ using the following unbiased gradient of $\rho(\pi_\btheta)$:
\[\frac{\partial}{\partial\btheta}\rho(\pi_\btheta) = \mathbf{E}\left[g(H) \sum_{t=0}^L \frac{\partial}{\partial\btheta} \log \pi_\btheta(A_t | S_t) \middle | H \sim \pi_\btheta \right]. \]
Intuitively, REINFORCE increases the probability of all actions taken during $H$ as a function of $g(H)$. 
This update increases the probability of actions that lead to high return trajectories.
BPG can be interpreted as REINFORCE where the return of a trajectory is the square of its importance-sampled return.
Thus BPG increases the probability of all actions taken along $H$ as a function of $IS(H,\btheta)^2$.
The magnitude of $IS(H,\btheta)^2$ depends on two qualities of $H$:
\begin{enumerate}
\item $g(H)^2$ is large (i.e., a high magnitude event).
\item $H$ is rare relative to its probability under the evaluation policy (i.e., $\prod_{t=0}^L \frac{\pi_e(A_t|S_t)}{\pi_\btheta(A_t|S_t)}$ is large).
\end{enumerate}

These two qualities demonstrate a balance in how BPG changes trajectory probabilities.
Increasing the probability of a trajectory under $\pi_\btheta$ will decrease $IS(H,\btheta)^2$ and so BPG increases the probability of a trajectory when $g(H)^2$ is large enough to offset the decrease in $IS(H,\btheta)^2$ caused by decreasing the importance weight.

\section{Empirical Study}

This section presents an empirical study of variance reduction through behavior policy search.
We design our experiments to answer the following questions:
\begin{itemize}
\item Can behavior policy search with BPG reduce policy evaluation MSE compared to on-policy estimates in both tabular and continuous domains?
\item Does adapting the behavior policy of the Doubly Robust estimator with DR-BPG lower the MSE of the Advantage Sum estimator?
\item Does the rarety of actions that cause high magnitude rewards affect the performance gap between BPG and Monte Carlo estimates?

\end{itemize}

\subsection{Experimental Set-up}

We address our first experimental question by evaluating BPG in three domains.
We briefly describe each domain here; full details are available in appendix C.

The first domain is a 4x4 Gridworld.
We obtain two evaluation policies by applying REINFORCE to this task, starting from a policy that selects actions uniformly at random.
We then select one evaluation policy, $\pi_1$, from the early stages of learning -- an improved policy but still far from converged -- and one after learning has converged, $\pi_2$.
We run all experiments once with $\pi_e:=\pi_1$ and a second time with $\pi_e:=\pi_2$.


Our second and third tasks are the continuous control Cart-pole Swing Up and Acrobot tasks implemented within RLLAB \cite{duan2016benchmarking}.
The evaluation policy in each domain is a neural network that maps the state to the mean of a Gaussian distribution.
Policies are partially optimized with trust-region policy optimization \cite{schulman2015trust} applied to a randomly initialized policy.

\subsection{Main Results}

\paragraph{Gridworld Experiments}

Figure \ref{fig:gridworld} compares BPG to Monte Carlo for both Gridworld policies, $\pi_1$ and $\pi_2$.
Our main point of comparison is the mean squared error (MSE) of both estimates at iteration $i$ over $100$ trials.
For $\pi_1$, BPG significantly reduces the MSE of on-policy estimates (Figure \ref{fig:gw-p1-mse}).
For $\pi_2$, BPG also reduces MSE, however, it is only a marginal improvement.

At the end of each trial we used the final behavior policy to collect $100$ more trajectories and estimate $\rho(\pi_e)$.
In comparison to a Monte Carlo estimate with $100$ trajectories from $\pi_1$, MSE is 85.48 \% lower with this improved behavior policy.
For $\pi_2$, the MSE is 31.02 \% lower.
This result demonstrates that BPG can find behavior policies that substantially lower MSE.


To understand the disparity in performance between these two instances of policy evaluation, we plot the distribution of $g(H)$ under $\pi_e$ (Figures \ref{fig:gw-p1-ret} and \ref{fig:gw-p2-ret}).
These plots show the variance of $\pi_1$ to be much higher; it sometimes samples returns with twice the magnitude of any sampled by $\pi_2$.
To quantify this difference, we also measure the variance of $\operatorname{IS}(H,\btheta_i)$ as $\mathbf{E}\left [\operatorname{IS}(H)^2 \middle | H \sim \pi_{\btheta_i} \right ] - \mathbf{E}\left[\operatorname{IS}(H)\middle | H \sim \pi_{\btheta_i}\right] ^2 $ where the expectations are estimated with 10,000 trajectories.
This evaluation is repeated 5 times per iteration and the reported variance is the mean over these evaluations.
The decrease in variance for each policy is shown in Figure \ref{fig:gw-p1-var}.
The high initial variance means there is much more room for BPG to improve the behavior policy when $\btheta_e$ is the partially optimized policy.

\begin{figure}[t!]

\centering
\begin{subfigure}{0.45\columnwidth}
\includegraphics[scale=0.12]{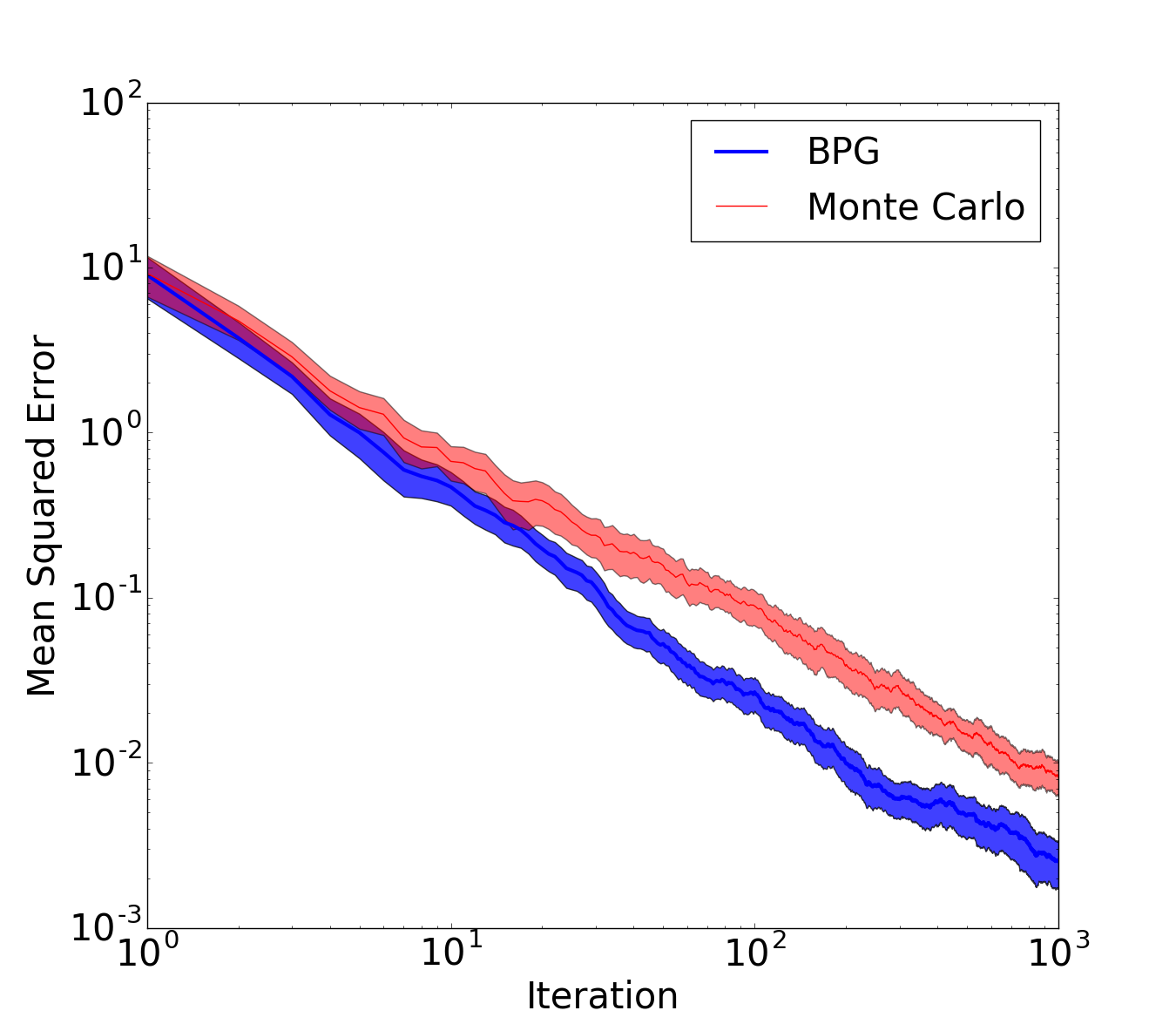}
\caption{Mean Squared Error}
\label{fig:gw-p1-mse}
\end{subfigure}
\begin{subfigure}{0.45\columnwidth}
\includegraphics[scale=0.12]{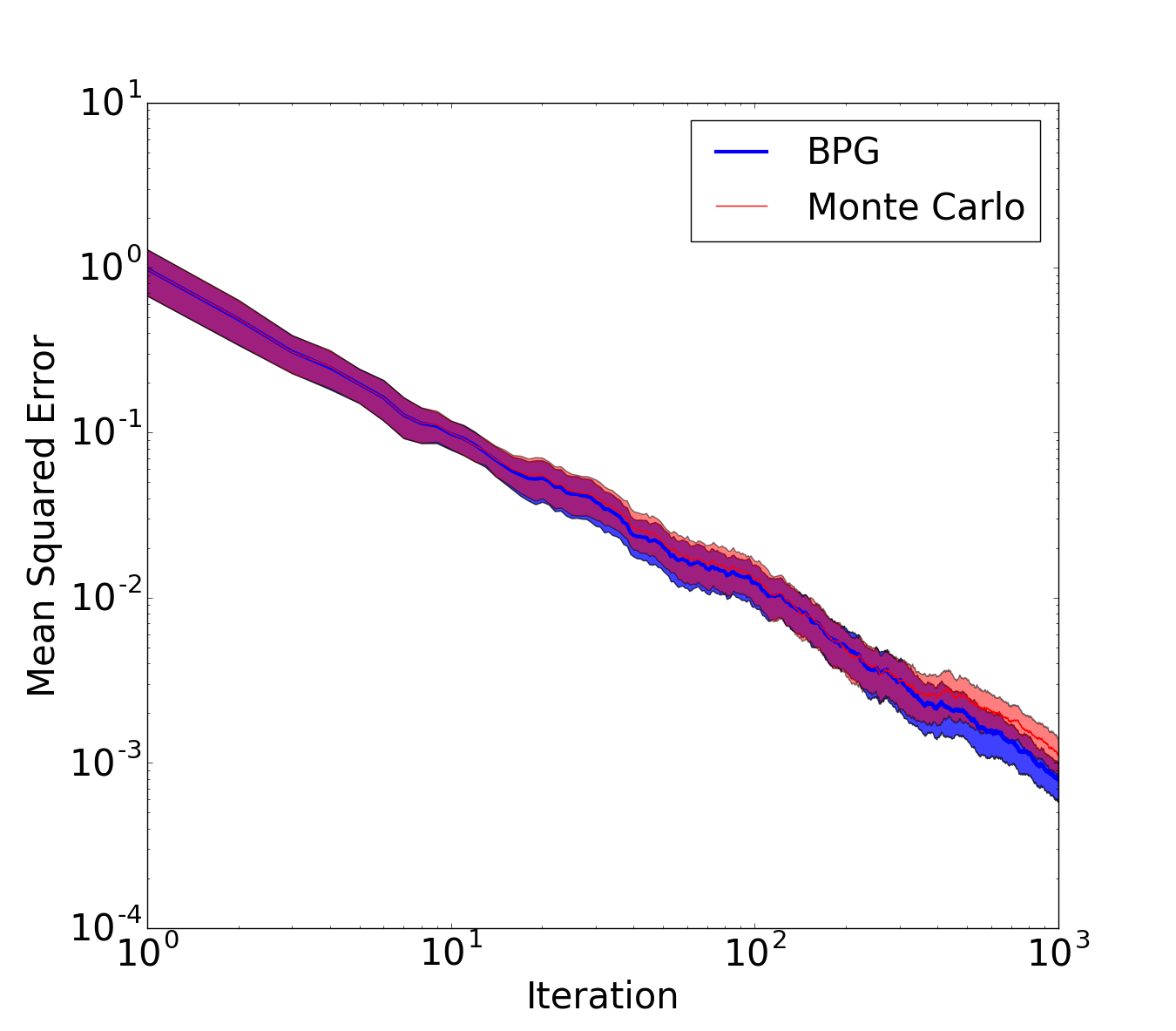}
\caption{Mean Squared Error}
\label{fig:gw-p2-mse}
\end{subfigure}

\centering
\begin{subfigure}{0.45\columnwidth}
\includegraphics[scale=0.12]{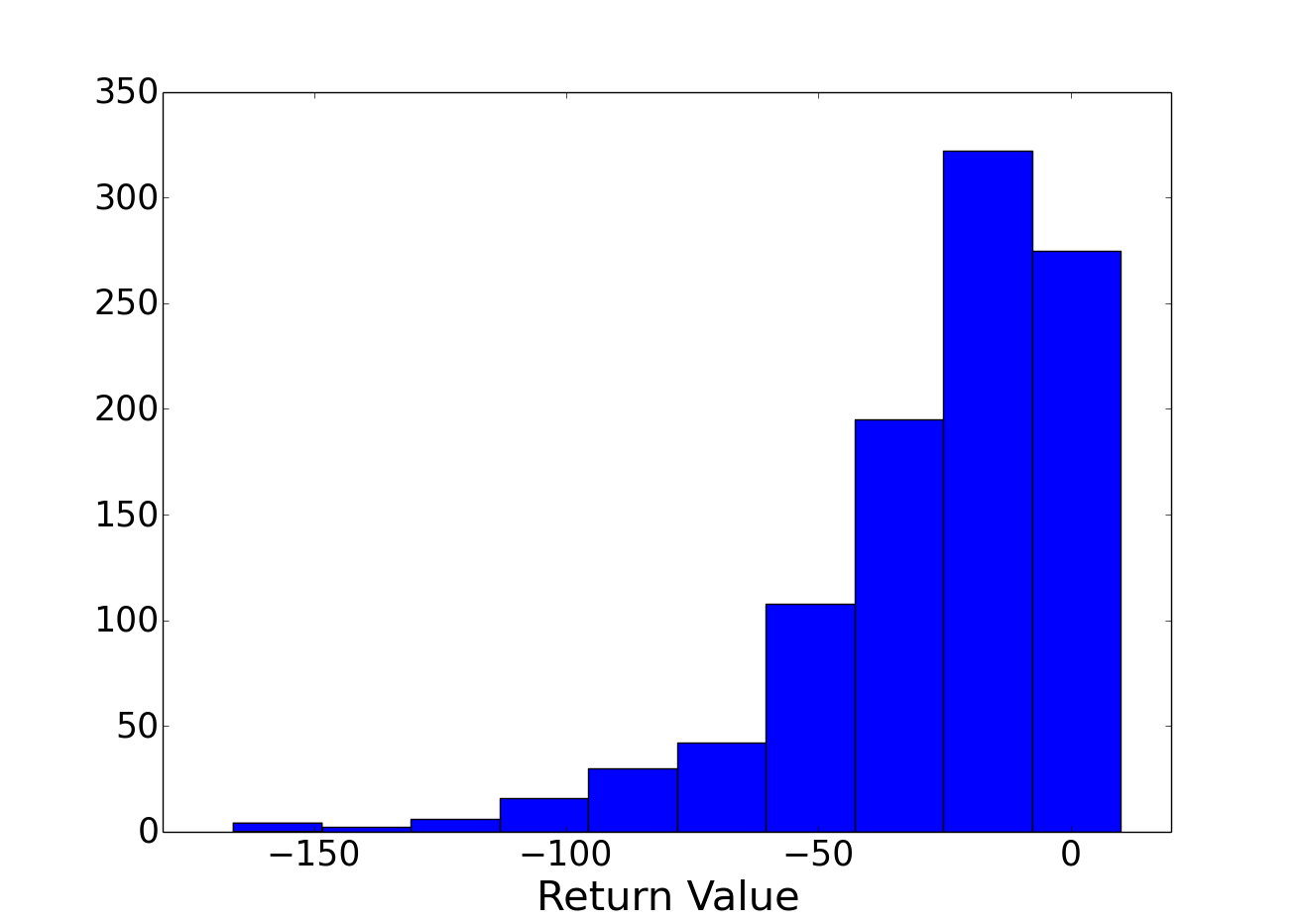}
\caption{Histogram of $\pi_1$ Returns}
\label{fig:gw-p1-ret}
\end{subfigure}
\begin{subfigure}{0.45\columnwidth}
\includegraphics[scale=0.12]{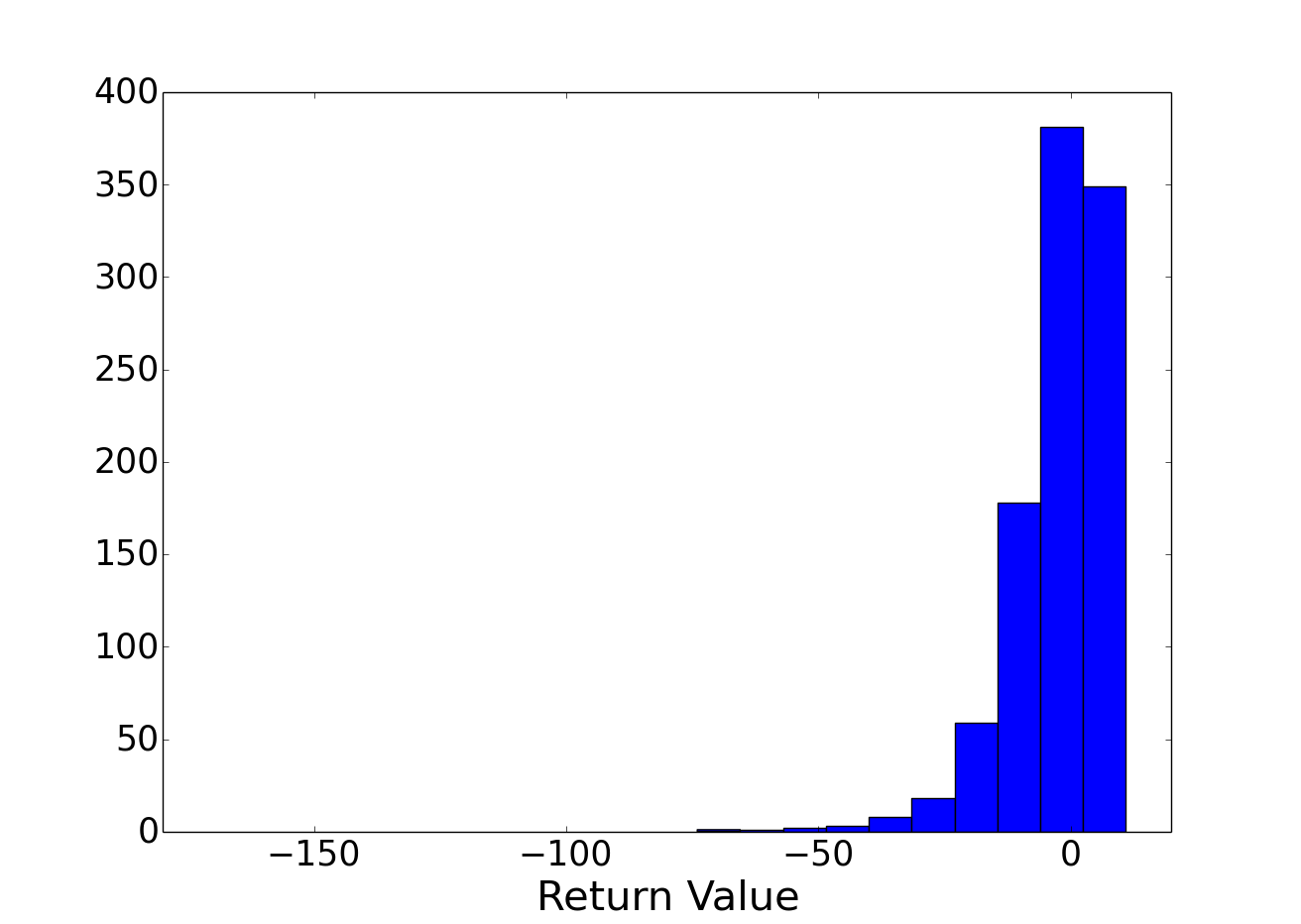}
\caption{Histogram of $\pi_2$ Returns}
\label{fig:gw-p2-ret}
\end{subfigure}

\centering
\begin{subfigure}{0.45\columnwidth}
\includegraphics[scale=0.12]{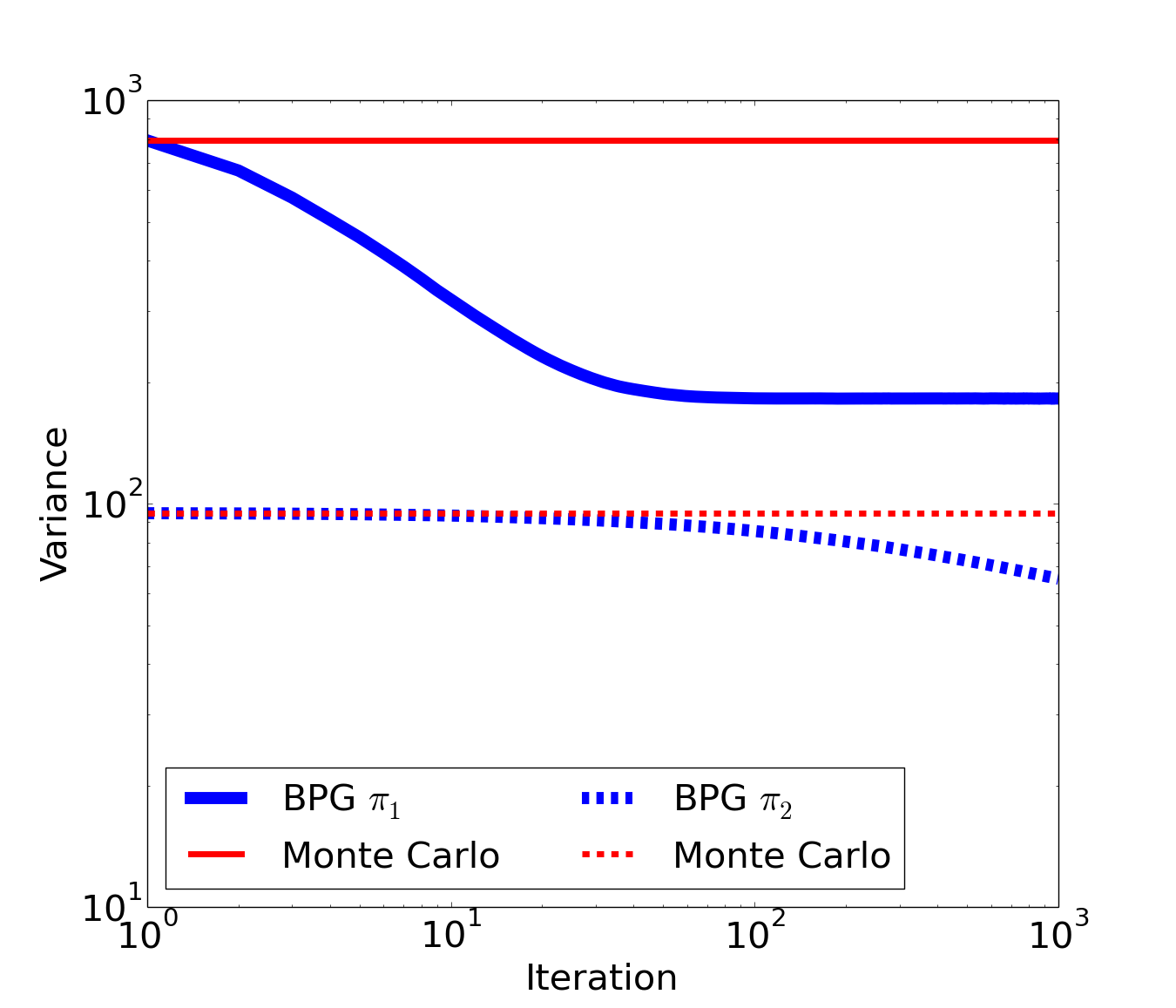}
\caption{Variance Reduction}
\label{fig:gw-p1-var}
\end{subfigure}
\begin{subfigure}{0.45\columnwidth}
\includegraphics[scale=0.12]{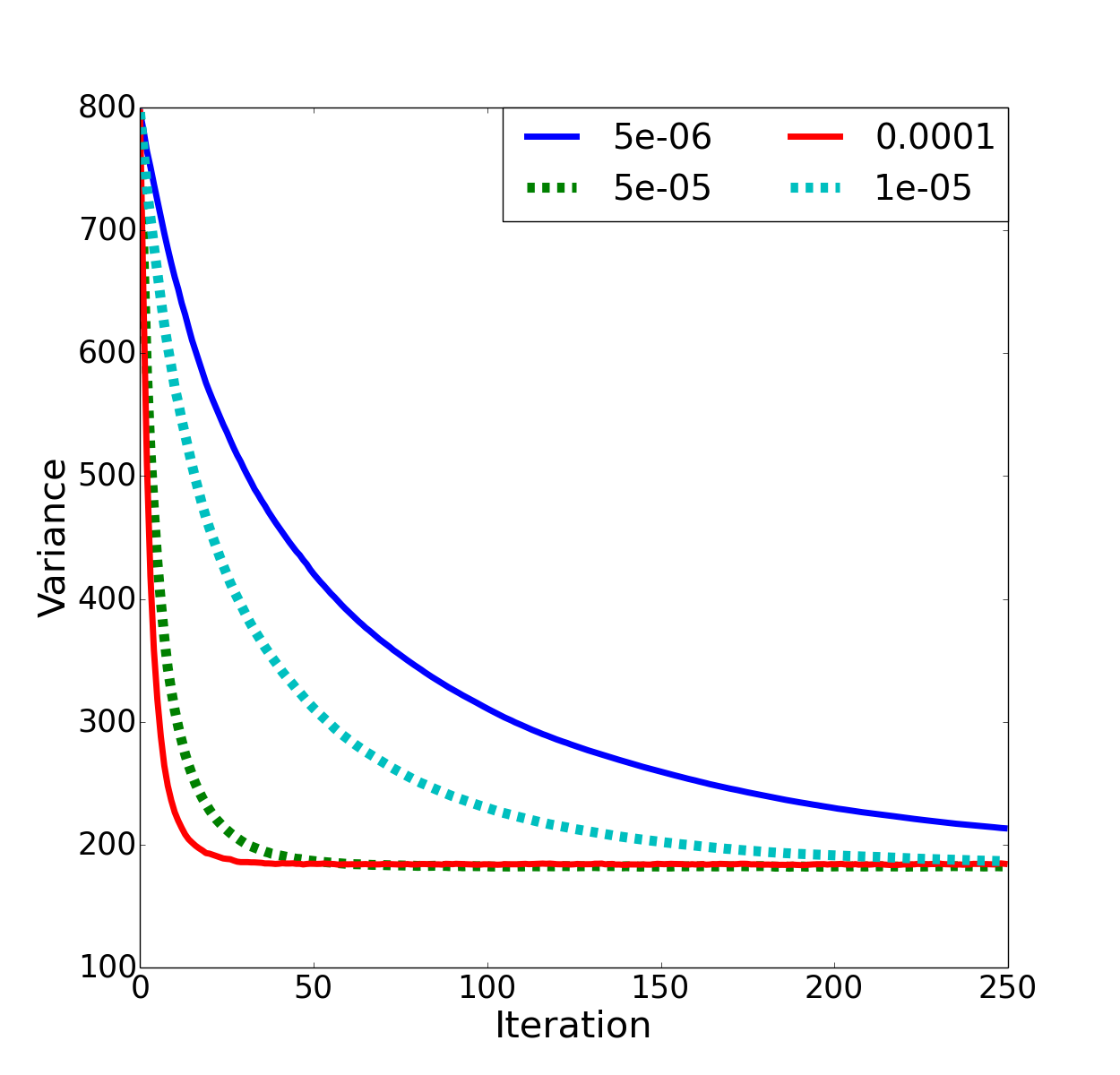}
\caption{Learning Rate Sensitivity}
\label{fig:gw-p2-var}
\end{subfigure}

\caption{Gridworld experiments when $\pi_e$ is a partially optimized policy, $\pi_1$, (\ref{fig:gw-p1-mse}) and a converged policy, $\pi_2$, (\ref{fig:gw-p2-mse}). The first and second rows give results for $\pi_1$ on the left and $\pi_2$ on the right. Results are averaged over $100$ trials of $1000$ iterations with error bars given for 95 \% confidence intervals. In both instances, BPG lowers MSE more than on-policy Monte Carlo returns (statistically significant, $p<0.05$). The second row shows the distribution of returns under the two different $\pi_e$. Figure \ref{fig:gw-p1-var} shows a substantial decrease in variance when evaluating $\pi_1$ with BPG and a lesser decrease when evaluating $\pi_2$ with BPG. Figure \ref{fig:gw-p2-var} shows the effect of varying the step-size parameter for representative $\alpha$ (BPG diverged for high settings of $\alpha$). We ran BPG for 250 iterations per value of $\alpha$ and averaged over 5 trials. Axes in \ref{fig:gw-p1-mse}, \ref{fig:gw-p2-mse}, and \ref{fig:gw-p1-var} are log-scaled.}

\label{fig:gridworld}
\end{figure}

We also test the sensitivity of BPG to the learning rate parameter.
A critical issue in the use of BPG is selecting the step size parameter $\alpha$.
If $\alpha$ is set too high we risk making too large of an update to $\btheta$ --- potentially stepping to a worse behavior policy.
If we are too conservative then it will take many iterations for a noticeable improvement over Monte Carlo estimation.
Figure \ref{fig:gw-p2-var} shows variance reduction for a number of different $\alpha$ values in the GridWorld domain.
We found BPG in this domain was robust to a variety of step size values.
We do not claim this result is representative for all problem domains; stepsize selection in the behavior policy search problem is an important area for future work.


\paragraph{Continuous Control}

Figure \ref{fig:swing} shows reduction of MSE on the Cartpole Swing-up and Acrobot domains.
Again we see that BPG reduces MSE faster than Monte Carlo evaluation.
In contrast to the discrete Gridworld experiment, this experiment demonstrates the applicability of BPG to the continuous control setting.
While BPG significantly outperforms Monte Carlo evaluation in Cart-pole Swing-up, the gap is much smaller in Acrobot.
This result also demonstrates BPG (and behavior policy search) when the policy must generalize across different states.

\begin{figure}
\begin{subfigure}{0.45\columnwidth}
\includegraphics[scale=0.12]{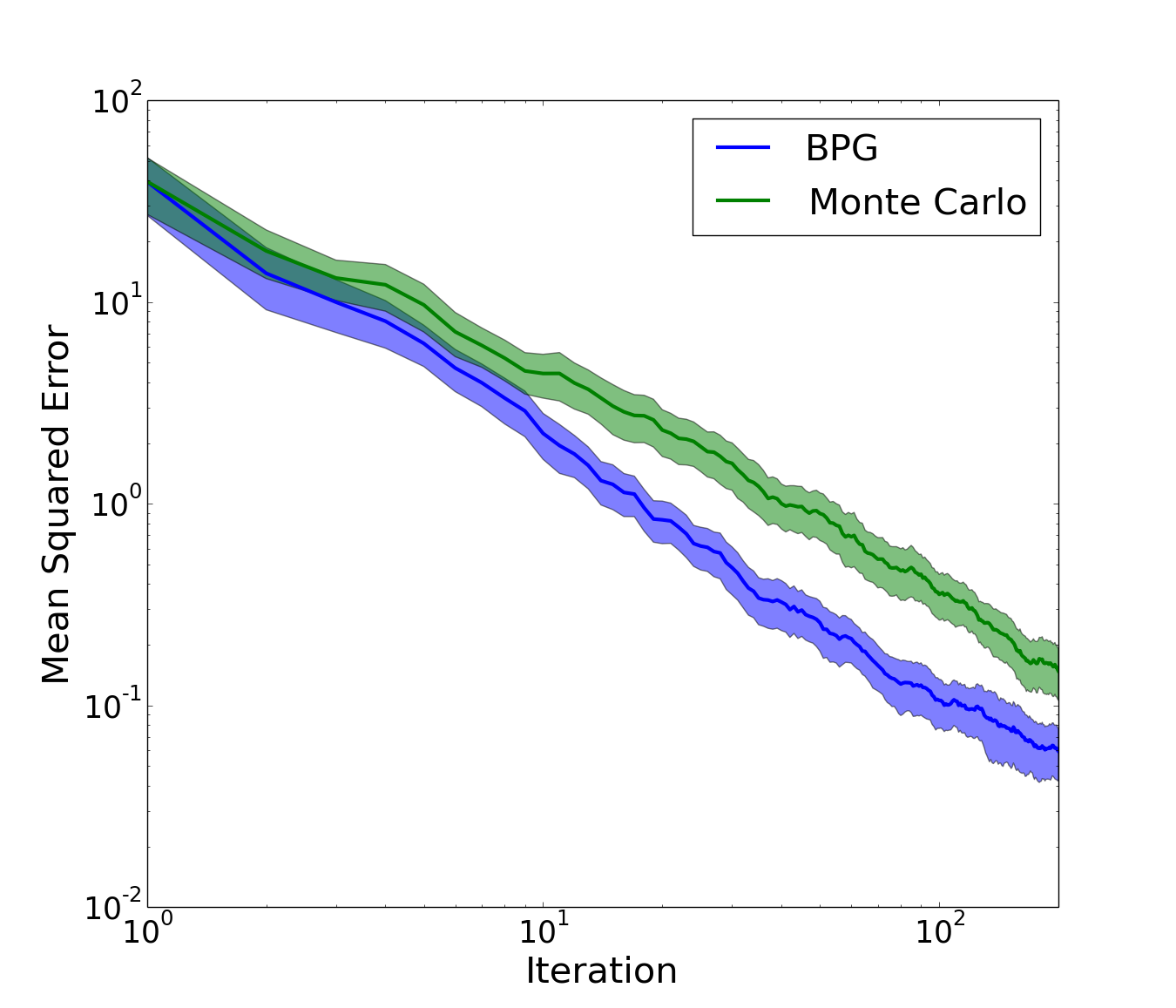}
\caption{Cart-pole Swing Up MSE}
\end{subfigure}
\begin{subfigure}{0.45\columnwidth}
\includegraphics[scale=0.12]{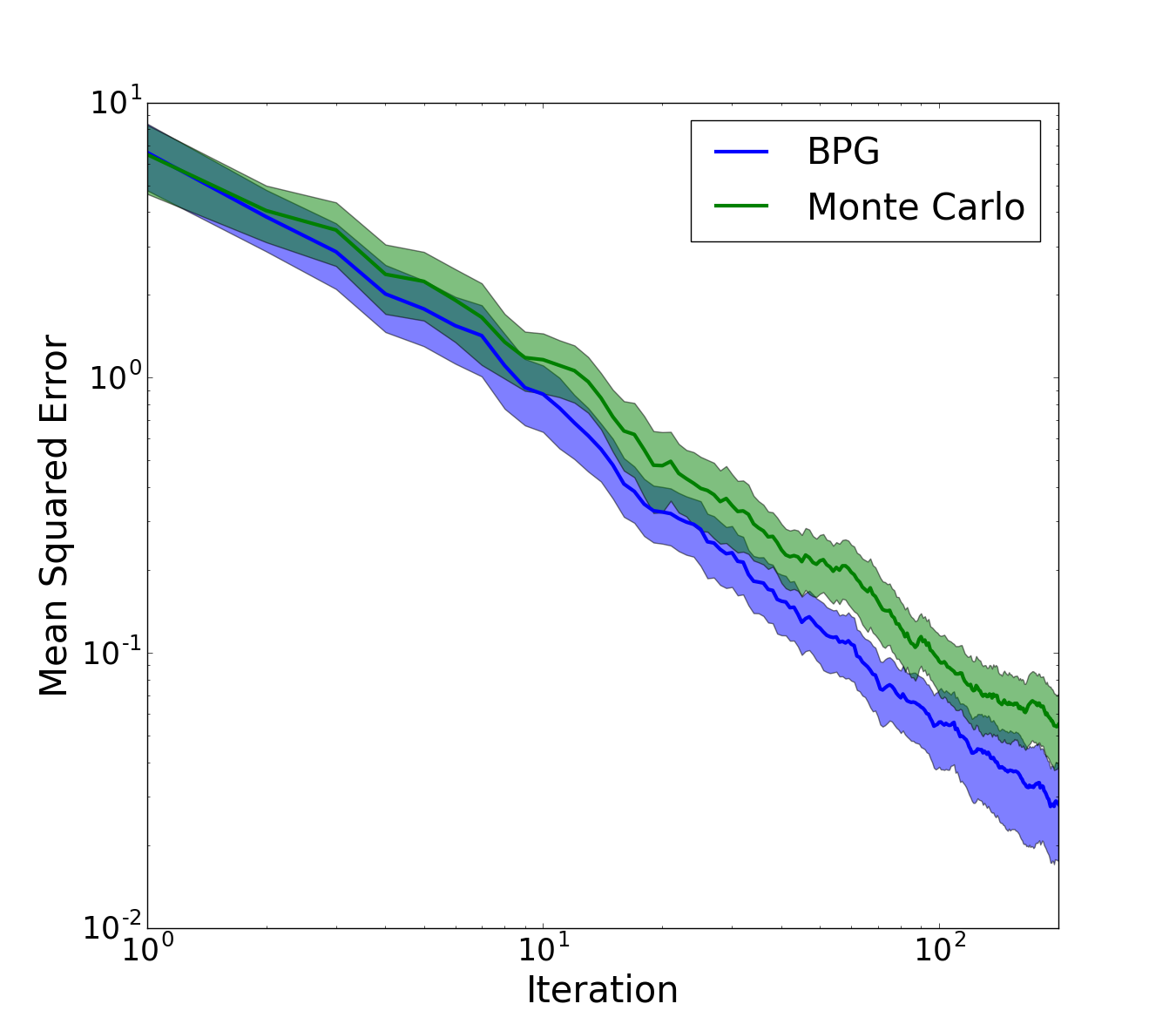}
\caption{Acrobot MSE}
\end{subfigure}
\caption{Mean squared error reduction on the Cart-pole Swing Up and Acrobot domains. 
We adapt the behavior policy for 200 iterations and average results over $100$ trials. Error bars are for 95\% confidence intervals.}
\label{fig:swing}
\vspace{-10pt}
\end{figure}


\subsection{Control Variate Extensions}

In this section, we evaluate the combination of model-based control variates with behavior policy search.
Specifically, we compare the AS estimator with Doubly Robust BPG (DR-BPG).
In these experiments we use a 10x10 stochastic gridworld. 
The added stochasticity increases the difficulty of building an accurate model from trajectories.

Since these methods require a model we construct this model in one of two ways.
The first method uses all trajectories in $\mathcal{D}$ to build the model and then uses the same set to estimate $\rho(\pi_e)$ with ASE or DR.
The second method uses trajectories from the first $10$ iterations to build the model and then fixes the model for the remaining iterations. 
For DR-BPG, behavior policy search starts at iteration $10$ under this second condition.
We call the first method ``update" and the second method ``fixed."
The update method invalidates the theoretical guarantees of these methods but learns a more accurate model.
In both instances, we build maximum likelihood tabular models.

\begin{figure}
\begin{subfigure}{0.45\columnwidth}
\includegraphics[scale=0.12]{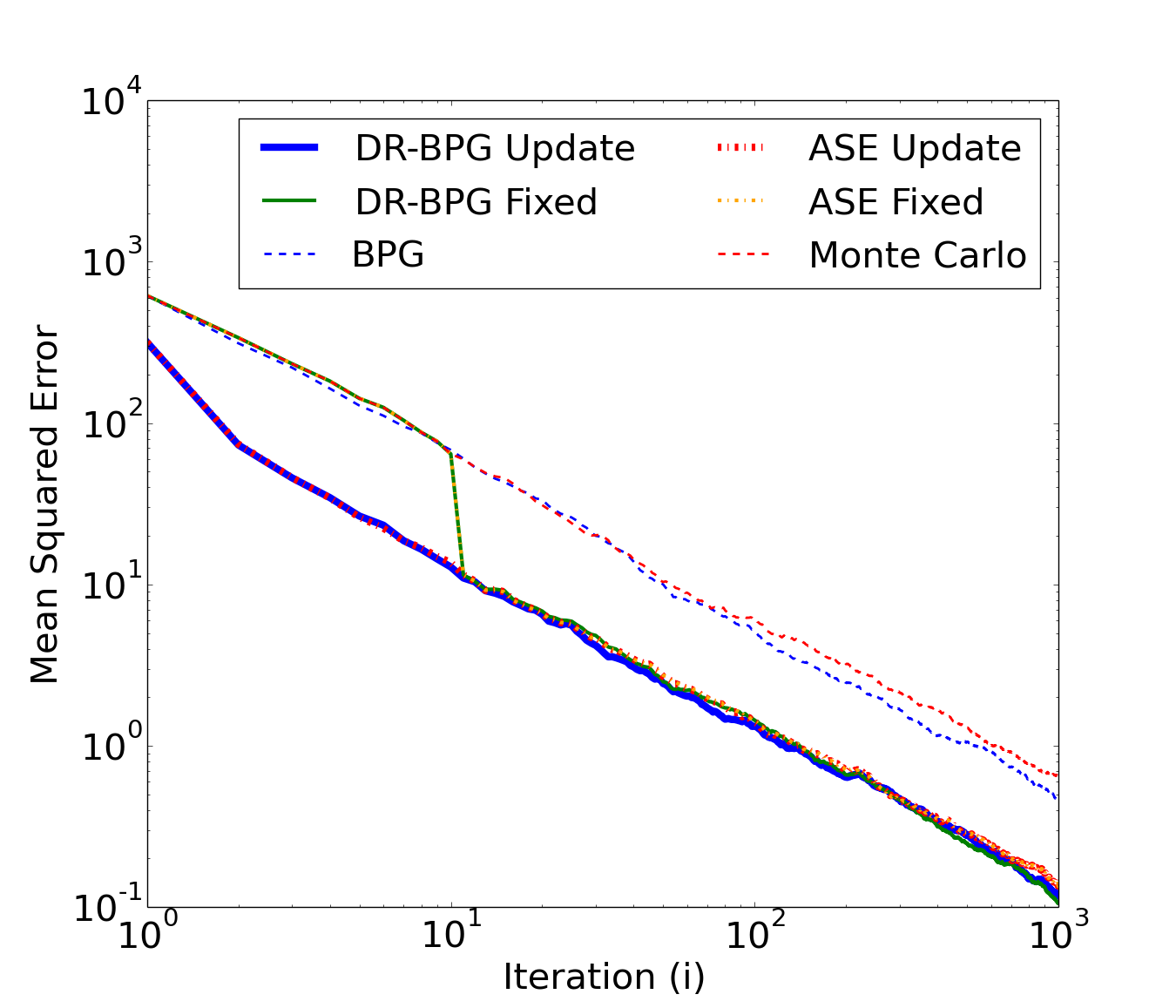}
\caption{Control Variate MSE}
\label{fig:cv}
\end{subfigure}
\begin{subfigure}{0.45\columnwidth}
\includegraphics[scale=0.12]{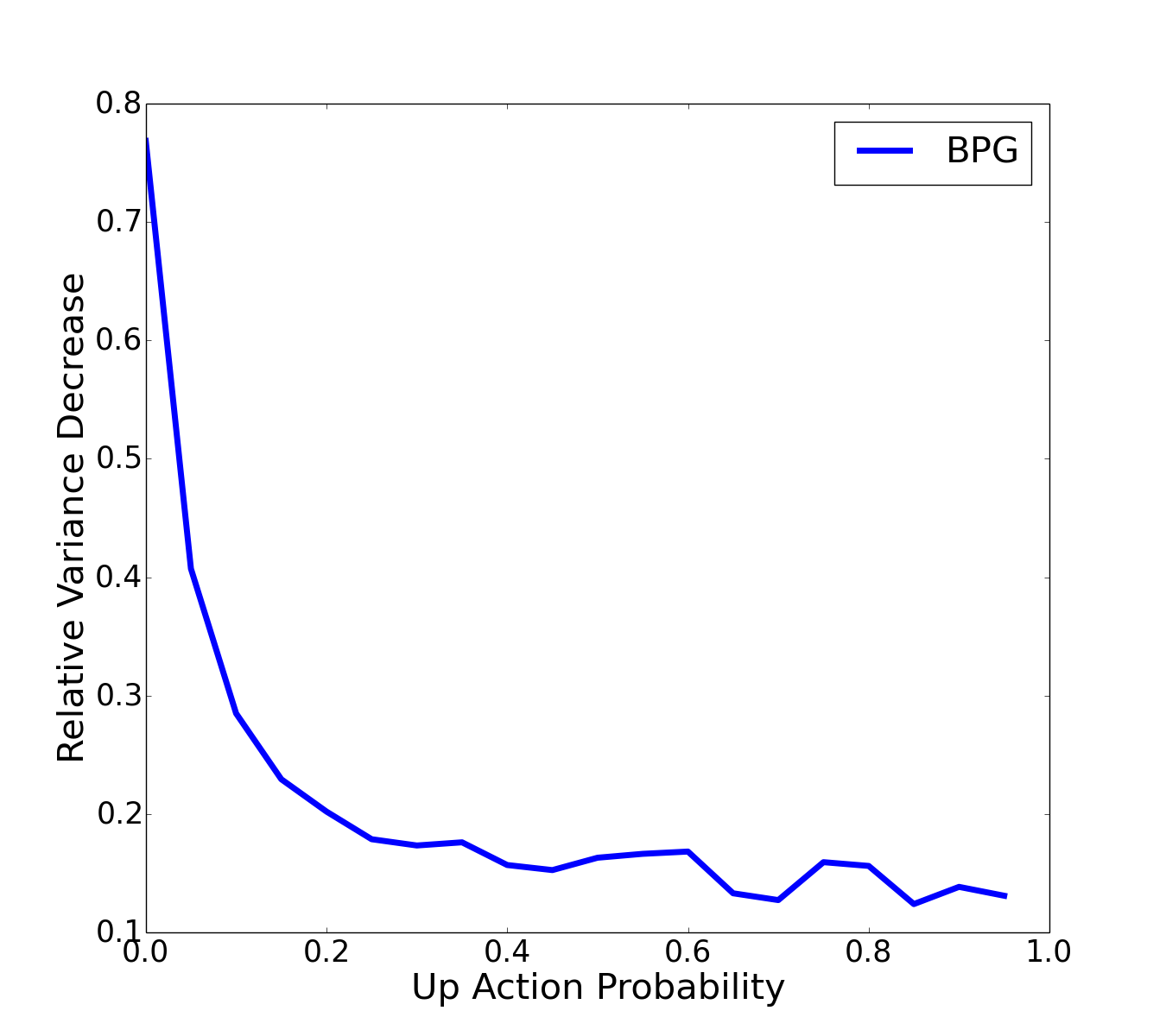}
\caption{Rare Event Improvement}
\label{fig:re}
\end{subfigure}
\caption{\textbf{\ref{fig:cv}:} Comparison of DR and ASE on a larger stochastic Gridworld. For the fixed model methods, the significant drop in MSE at iteration 10 is due to the introduction of the model control variate. For clarity we do not show error bars. The difference between the final estimate of DR-BPG and ASE with the fixed model is statistically significant ($p<0.05$); the difference between the same methods with a constantly improving model is not. \linebreak
\textbf{\ref{fig:re}:} Varying the probability of a high rewarding terminal action in the GridWorld domain. Each point on the horizontal axis is the probability of taking this action. The vertical axis gives the relative decrease in variance after adapting $\btheta$ for 500 iterations. Denoting the initial variance as $v_i$ and the final variance as $v_f$, the relative decrease is computed as $\frac{v_i - v_f}{v_i}$. Error bars for 95\% confidence intervals are given but are small. }
\label{fig:sgw-mb}
\vspace{-15pt}
\end{figure}

Figure \ref{fig:sgw-mb} demonstrates that combining BPG with a model-based control variate (DR-BPG) can lead to further reduction of MSE compared to the control variate alone (ASE).
Specifically, with the fixed model, DR-BPG outperformed all other methods.
DR-BPG using the update method for building the model performed competitively with ASE although not statistically significantly better.
We also evaluate the final learned behavior policy of the fixed model variant of DR-BPG.
For a batch size of $100$ trajectories, the DR estimator with this behavior policy improves upon the ASE estimator with the same model by 56.9 \%.

For DR-BPG, estimating the model with all data still allowed steady progress towards lower variance.
This result is interesting since a changing model changes the surface of our variance objective and thus gradient descent on the variance has no theoretical guarantees of convergence.
Empirically, we observe that setting the learning rate for DR-BPG was more challenging for either model type.
Thus while we have shown BPG can be combined with control variates, more work is needed to produce a robust method.

\subsection{Rareness of Event}

Our final experiment aims to understand how the gap between on- and off-policy variance is affected by the probability of rare events.
The intuition for why behavior policy search can lower the variance of on-policy estimates is that a well selected behavior policy can cause rare and high magnitude events to occur.
We test this intuition by varying the probability of a rare, high magnitude event and observing how this change affects the performance gap between on- and off-policy evaluation.
For this experiment, we use a variant of the deterministic Gridworld where taking the UP action in the initial state (the upper left corner) causes a transition to the terminal state with a reward of $+50$.
We use $\pi_1$ from our earlier Gridworld experiments but we vary the probability of choosing UP when in the initial state.
So with probability $p$ the agent will receive a large reward and end the trajectory.
We use a constant learning rate of $10^{-5}$ for all values of $p$ and run BPG for $500$ iterations.
We plot the relative decrease of the variance as a function of $p$ over 100 trials for each value of $p$.
We use relative variance to normalize across problem instances.
Note that under this measure, even when $p$ is close to $1$, the relative variance is not equal to zero because as $p$ approaches $1$ the initial variance also goes to zero.


This experiment illustrates that as the initial variance increases, the amount of improvement BPG can achieve increases.
As $p$ becomes closer to $1$, the initial variance becomes closer to zero and BPG barely improves over the variance of Monte Carlo (in terms of absolute variance there is no improvement).
When the $\pi_e$ rarely takes the high rewarding UP action ($p$ close to $0$), BPG improves policy evaluation by increasing the probability of this action.
This experiment supports our intuition for why off-policy evaluation can outperform on-policy evaluation.

\section{Related Work}

Behavior policy search and BPG are closely related to existing work on adaptive importance-sampling.
While adaptive importance-sampling has been studied in the estimation literature, we focus here on adaptive importance-sampling for MDPs and Markov Reward Processes (i.e., an MDP with a fixed policy).
Existing work on adaptive IS in RL has considered changing the transition probabilities to lower the variance of policy evaluation \cite{desai2001simulation,frank2008reinforcement} or lower the variance of policy gradient estimates \cite{ciosek2017offer}.
Since the transition probabilities are typically unknown in RL, adapting the behavior policy is a more general approach to adaptive IS.
Ciosek and Whiteson also adapt the distribution of trajectories with gradient descent on the variance \cite{ciosek2017offer} with respect to parameters of the transition probabilities.
The main focus of this work is increasing the probability of simulated rare events so that policy improvement can learn an appropriate response.
In contrast, we address the problem of policy evaluation and differentiate with respect to the (known) policy parameters.

The cross-entropy method (CEM) is a general method for adaptive importance-sampling.
CEM attempts to minimize the Kullback-Leibler divergence between the current sampling distribution and the optimal sampling distribution.
As discussed in Section \ref{sec:opt}, this optimal behavior policy only exists under a set of restrictive conditions.
In contrast we adapt the behavior policy by minimizing variance.

Other methods exist for lowering the variance of on-policy estimates.
In addition to the control variate technique used by the Advantage Sum estimator \cite{Zinkevich2006,White2009}, Veness et al. consider using common random numbers and antithetic variates to reduce the variance of roll-outs in Monte Carlo Tree Search (MCTS) \yrcite{Veness2011}. 
These techniques require a model of the environment (as is typical for MCTS) and do not appear to be applicable to the general RL policy evaluation problem.
BPG could potentially be applied to find a lower variance roll-out policy for MCTS.

In this work we have focused on unbiased policy evaluation. 
When the goal is to minimize MSE it is often permissible to use biased methods such as temporal difference learning \cite{van2014true}, model-based policy evaluation \cite{kearns2002near,strehl2009reinforcement}, or variants of weighted importance sampling \cite{precup2000eligibility}. 
It may be possible to use similar ideas to BPG to reduce bias and variance although this appears to be difficult since the bias contribution to the mean squared error is squared and thus any gradient involving bias requires knowledge of the estimator's bias.
We leave behavior policy search with biased off-policy methods to future work.



\section{Discussion and Future Work}


Our experiments demonstrate that behavior policy search with BPG can lower the variance of policy evaluation.
One open question is characterizing the settings where adapting the behavior policy substantially improves over on-policy estimates.
Towards answering this question, our Gridworld experiment showed that when $\pi_e$ has little variance, BPG can only offer marginal improvement.
BPG increases the probability of observing rare events with a high magnitude.
If the evaluation policy never sees such events then there is little benefit to using BPG.
However, in expectation and with an appropriately selected step-size, BPG will never lower the data-efficiency of policy evaluation.

It is also necessary that the evaluation policy contributes to the variance of the returns. 
If all variance is due to the environment then it seems unlikely that BPG will offer much improvement.
For example, Ciosek and Whiteson \yrcite{ciosek2017offer} consider a variant of the Mountain Car task where the dynamics can trigger a rare event --- independent of the action --- in which rewards are multiplied by $1000$. 
No behavior policy adaptation can lower the variance due to this event.

One limitation of gradient-based BPS methods is the necessity of good step-size selection.
In theory, BPG can never lead to worse policy evaluation compared to on-policy estimates.
In practice, a poorly selected step-size may cause a step to a worse behavior policy at step $i$ which may increase the variance of the gradient estimate at step $i+1$.
Future work could consider methods for adaptive step-sizes, second order methods, or natural behavior policy gradients.

One interesting direction for future work is incorporating behavior policy search into policy improvement.
A similar idea was explored by Ciosek and Whiteson who explored off-environment learning to improve the performance of policy gradient methods \yrcite{ciosek2017offer}.
The method presented in that work is limited to simulated environments with differential dynamics.
Adapting the behavior policy is a potentially much more general approach.

\section{Conclusion}

We have introduced the behavior policy search problem in order to improve estimation of $\rho(\pi_e)$ for an evaluation policy $\pi_e$.
We present a solution --- Behavior Policy Gradient --- for this problem which adapts the behavior policy with stochastic gradient descent on the variance of the importance-sampling estimator.
Experiments demonstrate BPG lowers the mean squared error of estimates of $\rho(\pi_e)$ compared to on-policy estimates.
We also demonstrate BPG can further decrease the MSE of estimates in conjunction with a model-based control variate method.

{\small
\section{Acknowledgements}

We thank Daniel Brown and the anonymous reviewers for useful comments on the work and its presentation.
This work has taken place in the Personal Autonomous Robotics Lab (PeARL) and Learning Agents Research Group (LARG) at the Artificial Intelligence Laboratory, The University of Texas at Austin. PeARL research is supported in part by NSF (IIS-1638107, IIS-1617639). LARG research is supported in part by NSF (CNS-1330072, CNS-1305287, IIS-1637736, IIS-1651089), ONR (21C184-01), AFOSR (FA9550-14-1-0087), Raytheon, Toyota, AT\&T, and Lockheed Martin. Josiah Hanna is supported by an NSF Graduate Research Fellowship. Peter Stone serves on the Board of Directors of Cogitai, Inc.  The terms of this arrangement have been reviewed and approved by the University of Texas at Austin in accordance with its policy on objectivity in research.
}

\onecolumn

\begin{appendix}
\section{Proof of Theorem 1}
In Appendix A, we give the full derivation of our primary theoretical contribution --- the importance-sampling (IS) variance gradient.
We also present the variance gradient for the doubly-robust (DR) estimator.

We first derive an analytic expression for the gradient of the variance of an arbitrary, unbiased off-policy policy evaluation estimator, $\operatorname{OPE}(H,\btheta)$.
Importance-sampling is one such off-policy policy evaluation estimator. From our general derivation we derive the gradient of the variance of the IS estimator and then extend to the DR estimator.

\subsection{Variance Gradient of an Unbiased Off-Policy Policy Evaluation Method}

We first present a lemma from which $\frac{\partial}{\partial\btheta}\operatorname{MSE}[IS(H,\btheta)]$ and $\frac{\partial}{\partial\btheta}\operatorname{MSE}[DR(H,\btheta)]$ can both be derived.

Lemma 1 gives the gradient of the mean squared error (MSE) of an unbiased off-policy policy evaluation method.

\begin{lemma}
{
\[\frac{\partial}{\partial\btheta}\operatorname{MSE}[\operatorname{OPE}(H,\btheta)]=  \\ \mathbf{E} \left [\operatorname{OPE}(H,\btheta)^2 (\sum_{t=0}^L \frac{\partial}{\partial\btheta} \log \pi_\btheta(A_t | S_t)) + \frac{\partial}{\partial \btheta} \operatorname{OPE}(H,\btheta)^2 \middle | H \sim \pi_\btheta \right ] 
\]
}
\end{lemma}

\begin{proof}

We begin by decomposing $\Pr(H|\pi)$ into two components---one that depends on $\pi$ and the other that does not. Let 
$$
w_\pi(H)\coloneqq \prod_{t=0}^{L} \pi(A_t|S_t),
$$
and
$$
p(H)\coloneqq \Pr(H|\pi) / w_\pi(H),
$$
for any $\pi$ such that $H \in \operatorname{supp}(\pi)$ (any such $\pi$ will result in the same value of $p(H)$). These two definitions mean that $\Pr(H|\pi)=p(H)w_\pi(H)$.

The MSE of the OPE estimator is given by:
\begin{align}
    \operatorname{MSE}[\operatorname{OPE}(H,\btheta)] = \operatorname{Var}[\operatorname{OPE}(H,\btheta)] + \underbrace{\left (\mathbf{E}[\operatorname{OPE}(H,\btheta)]-\rho(\pi_e) \right )^2}_{\text{bias}^2}.
\end{align}
Since the OPE estimator is unbiased, i.e., $\mathbf{E}[\operatorname{OPE}(H,\btheta)]=\rho(\pi_e)$, the second term is zero and so:
\begin{align}
    \operatorname{MSE}(\operatorname{OPE}(H,\btheta))=&\operatorname{Var}(\operatorname{OPE}(H,\btheta))\\
    =&\mathbf{E}\left [ \operatorname{OPE}(H,\btheta)^2 \middle | H \sim \pi_{\btheta} \right ] - \mathbf{E}[\operatorname{OPE}(H,\btheta) | H \sim \pi_{\btheta}  ]^2\\
    =& \mathbf{E}\left [ \operatorname{OPE}(H,\btheta)^2  \middle | H \sim \pi_{\btheta}  \right ] -\rho(\pi_e)^2 \label{eq:1}
    \end{align}
    %
    %

To obtain the $\operatorname{MSE}$ gradient, we differentiate $\operatorname{MSE}(\operatorname{OPE}(H,\btheta))$ with respect to $\btheta$:
\begin{align}
\frac{\partial}{\partial\btheta}\operatorname{MSE}[\operatorname{OPE}(H,\btheta)]=& \frac{\partial}{\partial\btheta} \left [ \mathbf{E} \left [\operatorname{OPE}(H,\btheta)^2 \middle | H \sim \pi_\btheta \right ] - \rho(\pi_e)^2 \right ] \\
=& \frac{\partial}{\partial\btheta}  \mathbf{E}_{H \sim \pi_\btheta} \left [\operatorname{OPE}(H,\btheta)^2 \right ] \\
=&\frac{\partial}{\partial\btheta} \sum_H \Pr(H|\btheta) \operatorname{OPE}(H,\btheta)^2 \\
=& \sum_H\Pr(H|\btheta) \frac{\partial}{\partial\btheta}  \operatorname{OPE}(H,\btheta)^2 + \operatorname{OPE}(H,\btheta)^2 \frac{\partial}{\partial\btheta} \Pr(H|\btheta)\\
=& \sum_H\Pr(H|\btheta) \frac{\partial}{\partial\btheta}  \operatorname{OPE}(H,\btheta)^2 + \operatorname{OPE}(H,\btheta)^2 p(H) \frac{\partial}{\partial\btheta} w_{\pi_\btheta}(H) \label{eq:lkajslkj}
\end{align}

    %
    %
    %
    %
Consider the last factor of the last term in more detail:
\begin{align}
    \frac{\partial}{\partial \btheta} w_{\pi_\btheta}(H) =& \frac{\partial}{\partial \btheta} \prod_{t=0}^{L} 
    \pi_\btheta(A_t|S_t)\\
    \overset{\text{(a)}}{=}& \left (\prod_{t=0}^{L} 
    \pi_\btheta(A_t|S_t) \right )\left ( \sum_{t=0}^{L} \frac{\frac{\partial}{\partial \btheta} \pi_{\btheta}(A_t|S_t)}{\pi_\btheta(A_t|S_t)}\right )\\
    =& w_{\pi_\btheta}(H) \sum_{t=0}^{L} \frac{\partial}{\partial \btheta} \log \left ( \pi_{\btheta}(A_t|S_t) \right ),
\end{align}
where {\bf (a)} comes from the multi-factor product rule. Continuing from \eqref{eq:lkajslkj} we have that:
\begin{align}
    \frac{\partial}{\partial \btheta} \operatorname{MSE}(\operatorname{OPE}(H,\btheta))
    =& \mathbf{E} \left [ \operatorname{OPE}(H,\btheta)^2 \sum_{t=0}^{L} \frac{\partial}{\partial \btheta} \log \left ( \pi_{\btheta}(A_t|S_t) \right ) + \frac{\partial}{\partial \btheta} \operatorname{OPE}(H,\btheta)^2\middle | H \sim \pi_{\btheta} \right ].
\end{align}


\end{proof}

\subsection{Behavior Policy Gradient Theorem}

We now use Lemma 1 to prove the Behavior Policy Gradient Theorem which is our main theoretical contribution.

\begin{thm}
\[ \frac{\partial}{\partial \btheta} \operatorname{MSE}[\operatorname{IS}(H,\btheta)] = \mathbf{E} \left [ -\operatorname{IS}(H,\btheta)^2 \sum_{t=0}^L \frac{\partial}{\partial\btheta}\log \pi_\btheta(A_t|S_t)\middle | H \sim \pi_\btheta \right]\]
where the expectation is taken over $H \sim \pi_\btheta$.
\end{thm}

\begin{proof}

We first derive $\frac{\partial}{\partial\btheta} \operatorname{IS}(H,\btheta)^2$. Theorem 1 then follows directly from using $\frac{\partial}{\partial\btheta} \operatorname{IS}(H,\btheta)^2$ as $\frac{\partial}{\partial\btheta} \operatorname{OPE}(H,\btheta)^2$ in Lemma 1.

\begin{align*}
\operatorname{IS}(H,\btheta)^2 =& \left (\frac{w_{\pi_e}}{w_\btheta}g(H) \right )^2 \\
 \frac{\partial}{\partial \btheta}\operatorname{IS}(H,\btheta)^2 =& \frac{\partial}{\partial \btheta} \left (\frac{w_{\pi_e}(H)}{w_\btheta(H)}g(H) \right )^2 \\
=& 2 \cdot g(H) \frac{w_{\pi_e}(H)}{w_\btheta(H)} \frac{\partial}{\partial \btheta} \left (g(H) \frac{w_{\pi_e}(H)}{w_\btheta(H)}\right ) \\
\overset{\text{(a)}}{=}& - 2 \cdot g(H) \frac{w_{\pi_e}(H)}{w_\btheta(H)} \left (g(H)\frac{w_{\pi_e}(H)}{w_\btheta(H)} \right )\sum_{t=0}^L \frac{\partial}{\partial \btheta} \log \pi_\btheta(A_t|S_t) \\
=& - 2 \operatorname{IS}(H,\btheta)^2 \sum_{t=0}^L \frac{\partial}{\partial \btheta} \log \pi_\btheta(A_t|S_t)
\end{align*}

where {\bf (a)} comes from the multi-factor product rule and using the likelihood-ratio trick (i.e., $\frac{\frac{\partial}{\partial\btheta}\pi_\btheta(A|S)}{\pi_\btheta(A|S)} = \log \pi_\btheta(A|S)$)

Substituting this expression into Lemma 1 completes the proof:

\[ \frac{\partial}{\partial \btheta} \operatorname{MSE}[\operatorname{IS}(H,\btheta)] = \mathbf{E} \left [ -\operatorname{IS}(H,\btheta)^2 \sum_{t=0}^L \frac{\partial}{\partial\btheta}\log \pi_\btheta(A_t|S_t) \middle | H \sim \pi_\btheta \right]\]

\end{proof}

\subsection{Doubly Robust Estimator}

Our final theoretical result is a corollary to the Behavior Policy Gradient Theorem: an extension of the IS variance gradient to the Doubly Robust (DR) estimator.
Recall that for a single trajectory DR is given as:
\[ 
\operatorname{DR}(H,\btheta) \coloneqq \hat{v}^{\pi_e}(S_0) + \sum_{t=0}^L \gamma^t \frac{w_{\pi_e,t}}{w_{\btheta,t}} \left (R_t - \hat{q}^{\pi_e}(S_t,A_t) + \hat{v}^{\pi_e}(S_{t+1})\right)
\]

where $\hat{v}^{\pi_e}$ is the state-value function of $\pi_e$ under an approximate model, $\hat{q}^{\pi_e}$ is the action-value function of $\pi_e$ under the model, and $w_{\pi,t}\coloneqq \prod_{j=0}^t \pi(A_j|S_j)$.

The gradient of the mean squared error of the DR estimator is given by the following corollary to the Behavior Policy Gradient Theorem:

\begin{cor}
{
\small
\begin{multline*}
\frac{\partial}{\partial\btheta}\operatorname{MSE}\left[DR(H,\btheta)\right] = \mathbf{E} [( \operatorname{DR}(H,\btheta)^2 \sum_{t=0}^L \frac{\partial}{\partial\btheta}\log \pi_\btheta(A_t|S_t)   -   2 \operatorname{DR}(H,\btheta)(\sum_{t=0}^L \gamma^t \delta_t\frac{w_{\pi_e,t}}{w_{\btheta,t}} \sum_{i=0}^t \frac{\partial}{\partial\btheta} \log \pi_\btheta (A_i|S_i))  ]
\end{multline*}
}
where $\delta_t = R_t - \hat{q}(S_t,A_t) + \hat{v}(S_{t+1})$ and the expectation is taken over $H \sim \pi_\btheta$.
\end{cor}


\begin{proof}

As with Theorem 1, we first derive $\frac{\partial}{\partial\btheta} \operatorname{DR}(H,\btheta)^2$. 
Corollary 1 then follows directly from using $\frac{\partial}{\partial\btheta} \operatorname{DR}(H,\btheta)^2$ as $\frac{\partial}{\partial\btheta} \operatorname{OPE}(H,\btheta)^2$ in Lemma 1.

\[ \operatorname{DR}(H,\btheta)^2 = \left(\hat{v}^{\pi_e}(S_0) + \sum_{t=0}^L \gamma^t \frac{w_{\pi_e,t}}{w_{\btheta,t}} \left (R_t - \hat{q}^{\pi_e}(S_t,A_t) + \hat{v}^{\pi_e}(S_{t+1})\right)\right)^2\]
 \begin{align*}
 \frac{\partial}{\partial \btheta}\operatorname{DR}(H,\btheta)^2 =& \frac{\partial}{\partial \btheta} \left ( \hat{v}^{\pi_e}(S_0) + \sum_{t=0}^L \gamma^t \frac{w_{\pi_e,t}}{w_{\btheta,t}} \left (R_t - \hat{q}^{\pi_e}(S_t,A_t) + \hat{v}^{\pi_e}(S_{t+1})\right)  \right )^2 \\
  =& 2 \operatorname{DR}(H,\btheta)\frac{\partial}{\partial \btheta} \left (\hat{v}^{\pi_e}(S_0) + \sum_{t=0}^L \gamma^t \frac{w_{\pi_e,t}}{w_{\btheta,t}} (R_t - \hat{q}^{\pi_e}(S_t,A_t) + \hat{v}^{\pi_e}(S_{t+1})) \right )  \\
 =& -2 \operatorname{DR}(H,\btheta)(\sum_{t=0}^L \gamma^t\frac{w_{\pi_e,t}}{w_{\btheta,t}}(R_t - \hat{q}^{\pi_e}(S_t,A_t) + \hat{v}^{\pi_e}(S_{t+1})) \sum_{i=0}^t \frac{\partial}{\partial\btheta} \log \pi_\btheta (A_i|S_i))
\end{align*}

Thus the $\operatorname{DR}(H,\btheta)$ gradient is:
{\small
\[ = \mathbf{E}\left [ \operatorname{DR}(H,\btheta)^2 \sum_{t=0}^L \frac{\partial}{\partial\btheta}\log \pi_\btheta(A_t|S_t) - 2 \operatorname{DR}(H,\btheta)(\sum_{t=0}^L \gamma^t \frac{w_{\pi_e,t}}{w_{\btheta,t}}(R_t - \hat{q}^{\pi_e}(S_t,A_t) + \hat{v}^{\pi_e}(S_{t+1})) \sum_{i=0}^t \frac{\partial}{\partial\btheta} \log \pi_\btheta (A_i|S_i)) \middle | H \sim \pi_\btheta  \right]\]
}

\end{proof}

The expression for the DR behavior policy gradient is more complex than the expression for the IS behavior policy gradient.
Lowering the variance of DR involves accounting for the covariance of the sum of terms.
Intuitively, accounting for the covariance increases the complexity of the expression for the gradient.

\section{BPG's Off-Policy Estimates are Unbiased}

This appendix proves that BPG's estimate is an unbiased estimate of $\rho(\pi_e)$.
If only trajectories from a single $\btheta_i$ were used then clearly $\operatorname{IS}(\cdot,\btheta_i)$ is an unbiased estimate of $\rho(\pi_e)$.
The difficulty is that the BPG's estimate at iteration $n$ depends on all $\btheta_i$ for $i=1\ldots n $ and each $\btheta_i$ is \textit{not} independent of the others.
Nevertheless, we prove here that BPG produces an unbiased estimate of $\rho(\pi_e)$ at each iteration.
Specifically, we will show that $\mathbf{E}\left[\operatorname{IS}(H_n, \btheta_n)\middle | \btheta_0 = \btheta_e)\right ]$ is an unbiased estimate of $\rho(\pi_e)$, 
 where the $\operatorname{IS}$ estimate is conditioned on $\btheta_0 = \btheta_e$.
 To make the dependence of $\btheta_i$ on $\btheta_{i-1}$ explicit, we will write $f(H_{i-1}) \coloneqq \btheta_i$ where $H_{i-1} \sim \pi_{\btheta_{i-1}}$. We use $\Pr(h|\btheta)$ as shorthand for $\Pr(H = h | \btheta)$.

\begin{align*}
\mathbf{E}\left[\operatorname{IS}(H_n,\btheta_n) \middle | \btheta = \btheta_e) \right] =& \sum_{h_0} \Pr(h_0 | \btheta_0) \sum_{h_1} \Pr(h_1 | f(h_0)) \cdots \underbrace{\sum_{h_n} \Pr(h_n | f(h_{n-1})) \operatorname{IS}(h_n)}_{\rho(\pi_e)} \\
=& \rho(\pi_e) \sum_{h_0} \Pr(h_0 | \btheta_0) \sum_{h_1} \Pr(h_1 | f(h_0)) \cdots \\
=& \rho(\pi_e)
\end{align*}

Notice that, even though BPG's off-policy estimates at each iteration are unbiased, they are \textit{not} statistically independent. 
This means that concentration inequalities, like Hoeffding's inequality, cannot be applied directly. 
We conjecture that the conditional independence properties of BPG (specifically that $H_i$ is independent of $H_{i-1}$ given $\theta_i$), are sufficient for Hoeffding's inequality to be applicable.

\section{Supplemental Experiment Description}

This appendix contains experimental details in addition to the details contained in Section 5 of the paper.

\paragraph{Gridworld:}
This domain is a 4x4 Gridworld with a terminal state with reward $10$ at $(3,3)$, a state with reward $-10$ at $(1,1)$, a state with reward $1$ at $(1,3)$, and all other states having reward $-1$.
The action set contains the four cardinal directions and actions move the agent in its intended direction (except when moving into a wall which produces no movement).
The agent begins in (0,0), $\gamma=1$, and $L=100$.
All policies use softmax action selection with temperature $1$ where the probability of taking an action $a$ in a state $s$ is given by:
\[ \pi(a|s) = \frac{e^{\theta_{sa}}}{\sum_{a^\prime} e^{\theta_{sa^\prime}}}\]
We obtain two evaluation policies by applying REINFORCE to this task, starting from a policy that selects actions uniformly at random.
We then select one evaluation policy from the early stages of learning -- an improved policy but still far from converged --, $\pi_1$, and one after learning has converged, $\pi_2$.
We run our set of experiments once with $\pi_e:=\pi_1$ and a second time with $\pi_e:=\pi_2$.
The ground truth value of $\rho(\pi_e)$ is computed with value iteration for both $\pi_e$.

\paragraph{Stochastic Gridworld:}
The layout of this Gridworld is identical to the deterministic Gridworld except the terminal state is at $(9,9)$ and the $+1$ reward state is at $(1,9)$.
When the agent moves, it moves in its intended direction with probability 0.9, otherwise it goes left or right with equal probability. 
Noise in the environment increases the difficulty of building an accurate model from trajectories.

\paragraph{Continuous Control:}
We evaluate BPG on two continuous control tasks: Cart-pole Swing Up and Acrobot.
Both tasks are implemented within RLLAB \cite{duan2016benchmarking} (full details of the tasks are given in Appendix 1.1).
The single task modification we make is that in Cart-pole Swing Up, when a trajectory terminates due to moving out of bounds we give a penalty of $-1000$.
This modification increases the variance of $\pi_e$.
We use $\gamma=1$ and $L=50$.
Policies are represented as conditional Gaussians with mean determined by a neural network with two hidden layers of 32 tanh units each and a state-independent diagonal covariance matrix. 
In Cart-pole Swing Up, $\pi_e$ was learned with 10 iterations of the TRPO algorithm \cite{schulman2015trust} applied to a randomly initialized policy.
In Acrobot, $\pi_e$ was learned with 60 iterations.
The ground truth value of $\rho(\pi_e)$ in both domains is computed with 1,000,000 Monte Carlo roll-outs.

\paragraph{Domain Independent Details}

In all experiments we subtract a constant control variate (or baseline) in the gradient estimate from Theorem 1. 
The baseline is $b_i = \mathbf{E} \left [-\operatorname{IS}(H)^2 \middle | H \sim \btheta_{i-1} \right ]$ and our new gradient estimate is:

\[
\mathbf{E}\left[ (-\operatorname{IS}^2 - b_i) \sum_{t=0}^L \frac{\partial}{\partial\btheta}\log \pi_\btheta(A_t|S_t) \middle | H \sim \pi_\btheta \right]
\]

Adding or subtracting a constant does not change the gradient in expectation since $b_i \cdot \mathbf{E}\left[\sum_{t=0}^L \frac{\partial}{\partial \btheta} \log \pi_\btheta (A_t | S_t)\right] = 0$.
BPG with a baseline has lower variance so that the estimated gradient is closer in direction to the true gradient.

We use batch sizes of $100$ trajectories per iteration for Gridworld experiments and size $500$ for the continuous control tasks.
The step-size parameter was determined by a sweep over $[10^{-2},10^{-6}]$

\paragraph{Early Stopping Criterion}

In all experiments we run BPG for a fixed number of iterations.
In general, BPS can continue for a fixed number of iterations or until the variance of the IS estimator stops decreasing.
The true variance is unknown but can be estimated by sampling a set of $k$ trajectories with $\btheta_i$ and computing the \textit{uncentered} variance:
$ \frac{1}{k}\sum_{j=0}^k \operatorname{OPE}(H_j,\btheta_j)^2$.
This measure can be used to empirically evaluate the quality of each $\btheta$ or determine when a BPS algorithm should terminate behavior policy improvement.

\end{appendix}


\end{document}